\RequirePackage[l2tabu, orthodox]{nag}
\documentclass[11pt]{article}

\usepackage[T1]{fontenc}

\usepackage[bitstream-charter]{mathdesign}
\usepackage{amsmath}
\usepackage[scaled=0.92]{PTSans}

\usepackage[
  paper  = letterpaper,
  left   = 1.65in,
  right  = 1.65in,
  top    = 1.0in,
  bottom = 1.0in,
  ]{geometry}

\usepackage[usenames,dvipsnames]{xcolor}
\definecolor{shadecolor}{gray}{0.9}

\usepackage[final,expansion=alltext]{microtype}
\usepackage[english]{babel}
\usepackage[parfill]{parskip}
\usepackage{afterpage}
\usepackage{framed}

\DeclareRobustCommand{\parhead}[1]{\textbf{#1}~}

\usepackage{lineno}

\usepackage{ragged2e}

\newcounter{parcount}

\usepackage{booktabs}
\usepackage{multirow}
\usepackage{tabularx}

\usepackage{listings}
\usepackage{fancyvrb}
\fvset{fontsize=\normalsize}

\usepackage{natbib}

\usepackage[colorlinks,linktoc=all]{hyperref}
\usepackage[all]{hypcap}
\hypersetup{citecolor=BurntOrange}
\hypersetup{linkcolor=black}
\hypersetup{urlcolor=MidnightBlue}

\usepackage[acronym,nowarn]{glossaries}

\lstdefinestyle{mystyle}{
    commentstyle=\color{OliveGreen},
    keywordstyle=\color{BurntOrange},
    numberstyle=\tiny\color{black!60},
    stringstyle=\color{MidnightBlue},
    basicstyle=\ttfamily,
    breakatwhitespace=false,
    breaklines=true,
    captionpos=b,
    keepspaces=true,
    numbers=left,
    numbersep=5pt,
    showspaces=false,
    showstringspaces=false,
    showtabs=false,
    tabsize=2
}
\usepackage[colorinlistoftodos,
           prependcaption,
           textsize=small,
           backgroundcolor=yellow,
           linecolor=lightgray,
           bordercolor=lightgray]{todonotes}

\DeclareRobustCommand{\parhead}[1]{\textbf{#1}~}

\usepackage{amsthm}  

\usepackage{bm}
\usepackage[Export]{adjustbox}

\usepackage[colorinlistoftodos,
           prependcaption,
           textsize=small,
           backgroundcolor=yellow,
           linecolor=lightgray,
           bordercolor=lightgray]{todonotes}

\usepackage{graphicx}
\usepackage{caption}
\usepackage{subcaption}
\usepackage{wrapfig}

\usepackage{booktabs}
\usepackage{arydshln} \usepackage{multirow}

\usepackage{listings}
\usepackage{fancyvrb}
\fvset{fontsize=\normalsize}

\hypersetup{linkcolor=black}
\hypersetup{urlcolor=BurntOrange}

\usepackage{listings}
\lstdefinestyle{alp_style}{
    commentstyle=\color{OliveGreen},
    numberstyle=\tiny\color{black!60},
    stringstyle=\color{BrickRed},
    basicstyle=\ttfamily\scriptsize,
    breakatwhitespace=false,
    breaklines=true,
    captionpos=b,
    keepspaces=true,
    numbers=none,
    numbersep=5pt,
    showspaces=false,
    showstringspaces=false,
    showtabs=false,
    tabsize=2
}

\usepackage{capt-of}

\usepackage{lipsum}

\usepackage{authblk}
\usepackage{enumitem}

\newtheorem{theorem}{Theorem}[section]

\theoremstyle{remark}
\newtheorem*{lemma*}{Lemma}

\usepackage{algorithm}
\usepackage{algpseudocode}
\usepackage{rotating}

\def\eqref#1{equation~\ref{#1}}

\def\1{\bm{1}}

\def\rvk{{\mathbf{k}}}

\def\rvp{{\mathbf{p}}}

\def\rvv{{\mathbf{v}}}

\def\rvx{{\mathbf{x}}}

\def\ervp{{\textnormal{p}}}

\def\rmI{{\mathbf{I}}}
\def\rmJ{{\mathbf{J}}}
\def\rmK{{\mathbf{K}}}

\DeclareMathAlphabet{\mathsfit}{\encodingdefault}{\sfdefault}{m}{sl}
\SetMathAlphabet{\mathsfit}{bold}{\encodingdefault}{\sfdefault}{bx}{n}

\def\gD{{\mathcal{D}}}

\def\gH{{\mathcal{H}}}

\def\gO{{\mathcal{O}}}

\def\gX{{\mathcal{X}}}

 \newacronym{ALI}{ali}{adversarially learned inference}
\newacronym{BIGAN}{bigan}{bidirectional generative adversarial network}
\newacronym{VI}{vi}{variational inference}
\newacronym{KL}{kl}{Kullback-Leibler}
\newacronym{ELBO}{elbo}{evidence lower bound}
\newacronym{MCMC}{mcmc}{Markov chain Monte Carlo}
\newacronym{HMC}{hmc}{Hamiltonian Monte Carlo}
\newacronym{RNN}{rnn}{recurrent neural network}
\newacronym{MLP}{mlp}{feed forward neural network}
\newacronym{GAN}{gan}{generative adversarial network}
\newacronym{DCGAN}{dcgan}{deep convolutional generative adversarial network}
\newacronym{PresGAN}{presgan}{prescribed generative adversarial network}
\newacronym{DGM}{dgm}{deep generative model}
\newacronym{PGAN}{pgan}{prescribed generative adversarial network}
\newacronym{VEEGAN}{veegan}{vee {GAN}}
\newacronym{PACGAN}{pacgan}{packed {GAN}}
\newacronym{STYLEGAN}{stylegan}{Style {GAN}}
\newacronym{FID}{fid}{{F}r\'{e}chet {I}nception distance}
\newacronym{IS}{is}{{I}nception score}
\newacronym{ML}{ml}{machine learning}
\newacronym{VS}{vs}{vendi score}
\newacronym{NLP}{nlp}{natural language processing}
\newacronym{IntDiv}{intdiv}{{I}nternal {D}iversity}
\newacronym{BLEU}{bleu}{BLEU}
\newacronym{PAIRWISE-BLEU}{pairwise-bleu}{PAIRWISE-BLEU}
\newacronym{D-LEX-SIM}{d-lex-sim}{D-LEX-SIM}
\newacronym{GILBO}{gilbo}{GILBO}
\newacronym{NOM}{nom}{number of modes}
\newacronym{HMM}{hmm}{HMM}
\newacronym{AAE}{aae}{AAE}
\newacronym{VAE}{vae}{VAE}
\newacronym{JTN}{jtn}{JTN}
\newacronym{Char-RNN}{char-rnn}{Char-RNN}
\newacronym{SMILES}{smiles}{SMILES}
\newacronym{MNIST}{mnist}{MNIST}
\newacronym{MultiNLI}{multinli}{MultiNLI}
\newacronym{StackedMNIST}{stackedmnist}{StackedMNIST}
\newacronym{NLI}{nli}{NLI}
\newacronym{VDVAE}{vdvae}{VDVAE}
\newacronym{LSUN}{lsun}{LSUN}
\newacronym{CIFAR}{cifar}{CIFAR} 
\newacronym{FD}{fd}{Fr\'echet Distance}
\newacronym{KD}{kd}{Kernel Distance}

\title{\textbf{Cousins Of The Vendi Score: A Family Of Similarity-Based Diversity Metrics For Science And Machine Learning}}

\author[1, 2]{Amey P. Pasarkar}
\author[1, 2, *]{Adji Bousso Dieng}
\affil[1]{Department of Computer Science, Princeton University}
\affil[2]{\href{https://vertaix.princeton.edu/}{Vertaix}}
\affil[*]{Published in \emph{Artificial Intelligence and Statistics, AISTATS 2024}}

\begin{document}
\maketitle

\begin{abstract}
  \noindent Measuring diversity accurately is important for many scientific fields, including \gls{ML}, ecology, and chemistry. The Vendi Score was introduced as a generic similarity-based diversity metric that extends the Hill number of order $q=1$ by leveraging ideas from quantum statistical mechanics. Contrary to many diversity metrics in ecology, the Vendi Score accounts for similarity and does not require knowledge of the prevalence of the categories in the collection to be evaluated for diversity. However, the Vendi Score treats each item in a given collection with a level of sensitivity proportional to the item's prevalence. This is undesirable in settings where there is a significant imbalance in item prevalence. In this paper, we extend the other Hill numbers using similarity to provide flexibility in allocating sensitivity to rare or common items. This leads to a family of diversity metrics---\emph{Vendi scores} with different levels of sensitivity controlled by the order $q$---that can be used in a variety of applications. We study the properties of the scores in a synthetic controlled setting where the ground truth diversity is known. We then test the utility of the Vendi scores in improving molecular simulations via Vendi Sampling. Finally, we use the scores to better understand the behavior of image generative models in terms of memorization, duplication, diversity, and sample quality\footnote{Code can be found at \url{https://github.com/vertaix/Vendi-Score}.}.\\
 
\end{abstract}

\section{INTRODUCTION}

\begin{figure*}[!hbpt]
    \includegraphics[width=\textwidth]{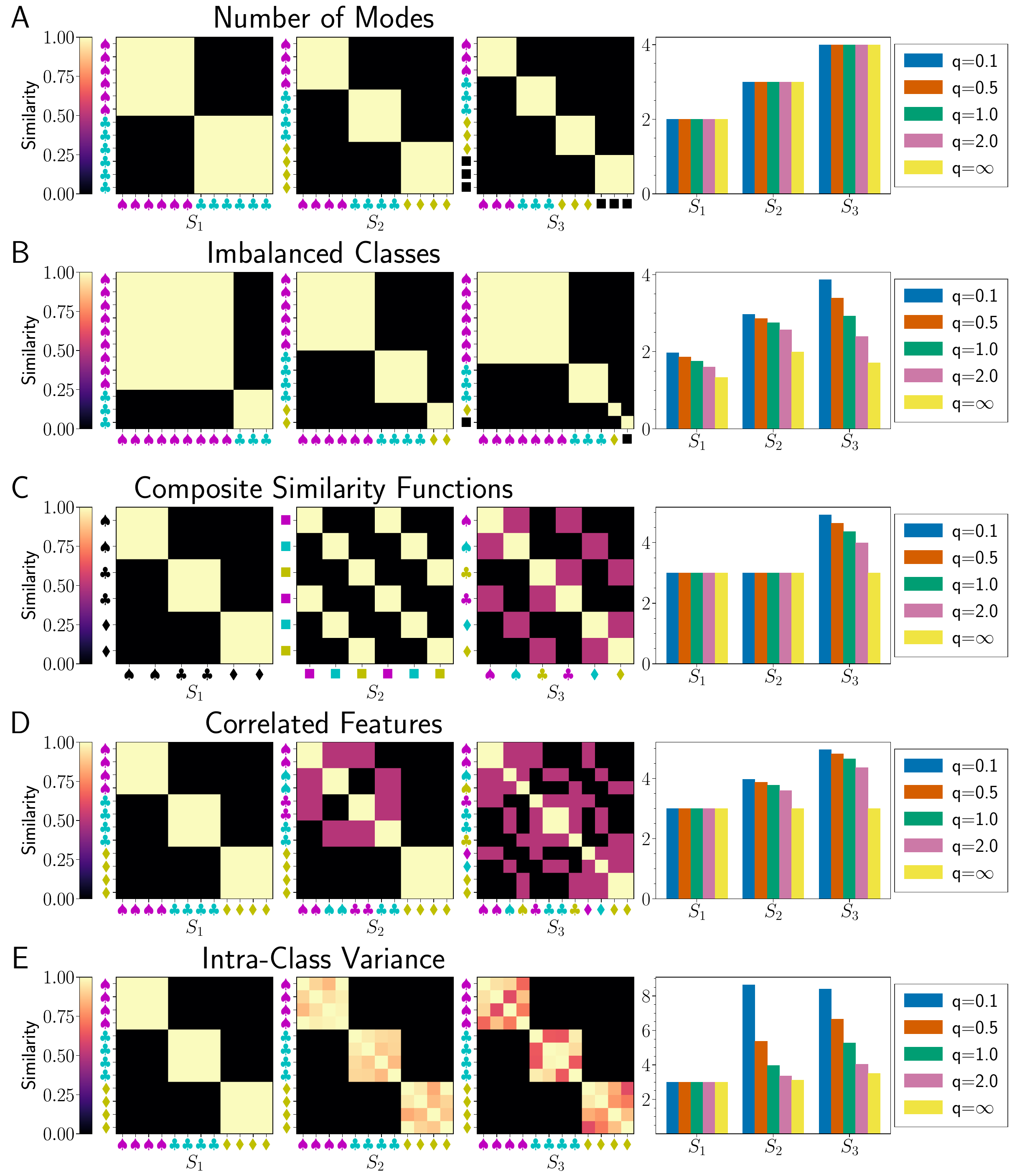}
    \caption{\textbf{Sensitivity of Different Vendi Scores Under Different Scenarios.} (A) Varying the number of classes under perfect balance. Each Vendi score measures the number of classes exactly; they are effective numbers. (B) Varying the number of classes under imbalance. Smaller orders $q$ more accurately describe the correct number of modes. (C) Combining two similarity functions for shape and color. All choices of order $q$ except $q=\infty$ give increases in diversity with the similarity composition. (D) Varying the correlation of shape and color features. As the correlation between shape and color decreases from left to right, all $q$ except $q=\infty$ yield larger Vendi scores. (E) Decreasing the similarity between class members. $q=\infty$ gives a Vendi Score that is more resistant to intra-class variance. For smaller $q$s, the Vendi scores increase with larger amounts of variance, although the Vendi score with $q=0.1$ decreases slightly between example $S_2$ and $S_3$. 
    }
    \label{fig:shape-color}
\end{figure*}

Evaluating diversity is a critical problem in many areas of machine learning (ML) and the natural sciences. Having a reliable diversity metric is necessary for evaluating generative models, curating datasets, and analyzing phenomena from the scale of molecules to evolutionary patterns. 

Ecologists have long studied the role of diversity in various ecosystems~\citep{whittaker1972evolution, hill1973diversity}, devising interpretable metrics that capture intuitive notions of diversity. However, these metrics tend to be limited in that they assume the ability to partition elements of an ecosystem into \emph{classes} or \emph{species} whose prevalence is known a priori. These metrics are also limited because they don't account for species similarity. Many have recently argued for the importance of accounting for species similarity to reliably measure diversity~\citep{leinster2012measuring}. We further argue that a diversity metric that accounts for similarity can be \emph{unsupervised}, i.e. such a metric doesn't need to assume the partitioning of elements of an ecosystem into known classes, nor does it need to assume knowledge of class prevalence. 

The Vendi Score was recently proposed as a generic unsupervised interpretable diversity metric that accounts for similarity by leveraging ideas from ecology and quantum mechanics~\citep{friedman2022vendi}. It's been shown useful for measuring the diversity of datasets and generative models~\citep{friedman2022vendi, stein2023exposing, diamantis2023intestine}, balancing the modes of image generative models~\citep{berns2023towards}, and accelerating molecular simulations~\citep{pasarkar2023vendi}. However, the Vendi Score accounts for different elements in a given collection according to their prevalence in the collection. This is undesirable in settings where there are large variations in item prevalence, such is the case for many \gls{ML} settings. We illustrate this failure mode in Figure \ref{fig:shape-color}, where the Vendi Score ($q=1$), under class imbalance, fails to separately account for the very rare classes (the black square and the yellow diamond) and lumps them into one class. 

\textbf{Contributions And Main Findings.} In this paper, we make several contributions and findings that we summarize below.
\begin{itemize}
    \item We extend the Vendi Score to a family of diversity metrics, \emph{Vendi scores} with different levels of sensitivity to item prevalence. The sensitivity is determined by a positive real number $q$, the \emph{order} of the score. The Vendi scores are based on the Hill numbers in ecology but, unlike Hill numbers, they account for similarity and are unsupervised.
    \item We showcase the usefulness of the Vendi scores in accelerating the simulation of Alanine Dipeptide, a well-studied benchmark molecular system. We find that the choice of $q$ can prioritize dynamics along certain axes, which can improve mixing and convergence.
    \item We show how the scores can be used to better evaluate and understand the behavior of generative models. We study the Vendi scores jointly with several metrics used in \gls{ML} to evaluate memorization, diversity, coverage, and sample quality. Our results reveal that generative models with a high human error rate or low \gls{FD} and \gls{KD}---i.e. those generative models that tend to produce samples that human evaluators cannot distinguish from real data---are those that memorize training samples and create duplicates around the memorized training samples. This finding calls for the need to pair sample quality metrics with a metric that reliably measures duplication or memorization and a metric that measures diversity effectively. We recommend pairing sample quality metrics with a Vendi score of small order $(q \in [0.1, 0.5])$ for diversity and the Vendi score of infinite order for duplication and memorization. Indeed, the Vendi score with infinite order is the most sensitive to duplicates and is strongly correlated with $C_T$-modified, a metric used to measure memorization, whereas Vendi scores of small order are more sensitive to rarer items and can effectively reflect diversity. 
    \item We found the scores to be strongly correlated, positively or negatively, with many existing metrics used to measure memorization and coverage. Those metrics rely on training data. Our findings suggest the Vendi scores provide the ability to indirectly evaluate memorization and coverage without relying on training data. This capability becomes even more important in privacy settings and as training datasets become more and more closed-source. 
\end{itemize}

\section{RELATED WORK}
\label{sec:related}

Several diversity metrics have been proposed in \gls{ML} and ecology. 

\parhead{Diversity metrics in \gls{ML}.} ML researchers often use some form of average pairwise similarity to quantify diversity, e.g. pairwise-BLEU~\citep{shen2019mixture} and D-Lex-Sim~\citep{fomicheva2020unsupervised} for text data or IntDiv for molecular data~\citep{benhenda2017chemgan}. Average similarity computations have been scaled from squared complexity to linear complexity in the size of the collection to be evaluated for diversity, enabling the assessment of the diversity of very large chemical databases~\citep{miranda2021extended, chang2022improving, racz2022molecular}. However, as discussed in \cite{friedman2022vendi}, average similarity can fail to effectively capture diversity, even in simple scenarios, e.g. it can score two populations with the same number of components/species but different levels of per-component variance the same. 

Other metrics used to evaluate diversity, especially in computer vision, include recall~\citep{sajjadi2018assessing} and Fr\'echet Inception distance (FID)~\citep{heusel2017gans}. However, these metrics are less flexible as they rely on a reference distribution, and in the case of FID, additionally require the availability of a pre-trained network. 

Yet other ways of measuring or enforcing diversity have been considered in active learning and experimental design settings~\citep{nguyen2023nonmyopic, maus2022discovering}. For example \cite{nguyen2023nonmyopic} enforce diversity via a diminishing returns criterion for multiclass active search, penalizing multiple explorations of the same class through a concave utility function. This yields improved results in multiclass active search. However, the approach is domain-specific and targets diversity indirectly. Several other works have studied diversity in the framework of Bayesian optimization~\citep{maus2022discovering} or evolutionary algorithms~\citep{mouret2015illuminating, vassiliades2017using, pugh2016quality}. 

\parhead{Diversity metrics in Ecology.} Ecologists have long been interested in quantifying diversity and have developed several metrics to assess the diversity of ecological systems. Some of the most ubiquitously used metrics in ecology are arguably the Hill numbers~\citep{hill1973diversity} and the triplets alpha diversity, beta diversity, and gamma diversity~\citep{whittaker1972evolution}. 

Hill numbers have been shown to be the only family of diversity metrics that satisfy the axioms of diversity~\citep{leinster2012measuring}. Hill numbers, although interpretable and grounded in intuitive notions of diversity, have important shortcomings that limit their use: (1) they assume some concept of classes and an ability to classify samples within classes (2) they assume knowledge of an abundance vector $p$ quantifying the number of elements in each class and finally (3) they ignore similarity between elements. 

Gamma diversity measures the total diversity of an ecosystem spanning some space. \cite{whittaker1972evolution} intuited that such a diversity metric should account for both the \emph{local} diversity measured over individual sites spanning a narrower region of the space (or \emph{alpha diversity}) and the differentiation among the different sites (or \emph{beta diversity}). These diversity indices have the same limitations as the Hill numbers mentioned above. 

\parhead{The Vendi Score.} The Vendi Score aims to alleviate many of the challenges faced by the commonly employed metrics in \gls{ML} and ecology. It is interpretable, reference-free, and satisfies the same axioms of diversity as the Hill numbers~\citep{friedman2022vendi}. Furthermore, unlike the Hill numbers, the Vendi Score accounts for similarity and doesn't require knowledge of class prevalence. 

In a recent extensive evaluation study for image generative models, \cite{stein2023exposing} used the Vendi score of order $q=1$ and found it to work more effectively as a measure of per-class diversity. In this paper, we show that by using different orders $q$ of the Vendi score, we can gain useful insights into the global diversity of generative model outputs.

\section{HILL NUMBERS AND ECOLOGICAL DIVERSITY}
\label{sec:hill}

Consider a probability distribution $\rvp = (\ervp_1, \dots, \ervp_S)$ on a space $\gX = \left\{1, \dots, S\right\}$. Ecologists refer to each member of $\gX$ as a \emph{species} and to the individual probability $p_i$ as the \emph{relative abundance} of the $i^{\text{th}}$ species in $\gX$. Ecologists have proposed a number of axioms that a diversity metric should satisfy to match intuitions~\cite{leinster2012measuring}:
\begin{enumerate}
    \item \textbf{Effective number.} Diversity is defined as the effective number of species in a population, ranging between $1$ and $\vert\gX\vert$. A population containing $N$ equally abundant, completely dissimilar species should have a diversity score of $N$. If all species are identical, the diversity should be minimized and equal to $1$.
    \item \textbf{Partitioning.} Suppose a population is partitioned into subpopulations, with no species shared between subsets, and the species in each subset completely dissimilar from the species in any other subset. Then the diversity of $\gX$ should be entirely determined by the diversity and size of each subpopulation.
    \item \textbf{Identical species.} If two species are identical, then merging them into one should leave the diversity of the population unchanged. 
    \item \textbf{Monotonicity.} When the similarities between species are increased, diversity should decrease. 
    \item \textbf{Permutation symmetry.} Diversity should be unchanged by changing the order in which the species are listed. 
\end{enumerate}

Historically, most ecological diversity indices have not accounted for species similarity, making the assumption that all species are completely dissimilar and defining diversity only in terms of the relative abundance $\rvp$. In this setting, Chapter 7 of \cite{leinster2020entropy} shows that the only metrics satisfying the axioms described above are the \emph{Hill numbers}. The Hill number $\gD_q$ of order $q$ is the exponential of the Renyi entropy $\gH_q$ of order $q$,

\begin{align} 
    \gH_q(\rvp) &= \frac{1}{1 - q}\log\sum_{i \in \text{supp}(\rvp)} p_i^q \quad \text{and} \quad 
    \gD_q(\rvp) = \exp\left(\gH_q(\rvp)\right) \label{def:hill}
    .
\end{align}
 
Here $\text{supp}(\rvp)$ denotes the set of indices $i$ for which $p_i > 0$ and $q \geq 0$ determines the relative weight assigned to rare or common items. With $q = 0$, all species are given equal weight and $\gD_0(\rvp)$ is equal to the size of the support. This is an uninformative measure of diversity. More interesting diversity indices correspond to $q \ne 0$, with $q = 1$ and $q = \infty$ corresponding to special limit cases. Indeed,  the Hill number of order $q = 1$ is the exponential of the Shannon entropy of $\rvp$,
\begin{align}
    \gD_1(\rvp) &= \exp\left(-\sum_{i\in \text{supp}(\rvp)} p_i \log p_i\right)
\end{align}
and weighs each species in proportion to its prevalence. The other interesting Hill number is also a limit, the Hill number of infinite order, 
\begin{align}
    \gD_{\infty}(\rvp) &= \exp(-\log\max_i p_i) = \frac{1}{\max_i p_i}
\end{align}
which assigns all the weight to the most common species. For $q \not\in \left\{0, 1, \infty\right\}$, the behavior depends on whether $q$ is less than $1$ or greater than $1$. Values of $q$ smaller than $1$ assign higher weight to rare species whereas large values of $q$ assign higher weight to common species. 

Despite their popularity in ecology, Hill numbers have shortcomings that limit their use beyond ecology: they make the strong assumption that species prevalence is known and don't account for species similarity.

\section{COUSINS OF THE VENDI SCORE: EXTENDING HILL NUMBERS USING SIMILARITY}
\label{sec:cousins}

How can we lift the limitations of the Hill numbers mentioned above and extend their applicability? ~\cite{friedman2022vendi} provide a solution for $q = 1$, drawing ideas from quantum statistical mechanics. Indeed, the von Neumann entropy $\mathcal{H}(\rho)$ for a quantum system with density matrix $\rho$ is of the same form as the Hill number of order $1$, 
\begin{align}
    \mathcal{H}(\rho) &= -\text{tr}(\rho\log \rho) = -\sum_i \lambda_i\log\lambda_i 
\end{align}
where the $\lambda_i$s are the eigenvalues of $\rho$. Replacing the density matrix with a normalized similarity matrix of species yields the Vendi Score:
\begin{align}
    \text{VS}(\rvx, \mathbf{k}) &= \exp\left(-\text{tr}\left(\frac{K_{\rvx}}{N}\log \frac{K_{\rvx}} {N}\right)\right) = \exp\left(- \sum_i \lambda_i\log\lambda_i \right)
\end{align}

where $\mathbf{k}(\cdot, \cdot)$ is a user-defined similarity function that induces a similarity matrix $K_{\rvx}$ over the species and the $\lambda_i$s are the eigenvalues of $\frac{K_{\rvx}} {N}$.

In this paper, we provide a theorem that relates the eigenvalues of a normalized similarity matrix to item prevalence and use this result to extend the Vendi Score to the other Hill numbers. 

\begin{theorem}\label{thm}[The Similarity-Eigenvalue-Prevalence Theorem]\label{thm:sim-eigen-prev}
    Let $(\rvx_1, \dots, \rvx_N)$ denote a collection of elements, where each $\rvx_i = (\rvx_{i1}, \dots, \rvx_{iM_i})$ contains a unique element repeated $M_i$ times, i.e. $\rvx_{ij} = \rvx_{ik}$ for all $j,k \in \left\{1, \dots, M_i\right\}$. Define $C = \sum_{i=1}^{N} M_i$. Let $\rmK \in \mathbb{R}^{C\times C}$ denote a kernel matrix such that $\rmK(\rvx_{i\bullet}, \rvx_{j\bullet}) = 1$ when $i = j$ and $0$ otherwise, $\forall i,j \in \left\{1, \dots, N\right\}$. Denote by $\tilde{\rmK} = \frac{\rmK}{C}$ the normalized kernel. Then $\tilde{\rmK}$ has exactly $N$ non-zero eigenvalues $\lambda_1, \dots, \lambda_N$ and $\lambda_i = \frac{M_i}{C}$ \text{ } $\forall i \in \left\{1, \dots, N\right\}$. 
\end{theorem}
\begin{proof}
    A complete proof of this theorem can be found in the appendix. 
\end{proof}

The theorem states that under the assumption that members of different species are completely dissimilar---this is the same assumption made in the computation of the Hill numbers--- there are as many nonzero eigenvalues of the species similarity matrix as there are species and that these eigenvalues are exactly equal to the prevalence of the different species. This theorem therefore provides a recipe for recovering the different Hill numbers exactly using the similarity-based construct the Vendi Score is based on. 

The benefit of computing the Hill numbers using the same similarity-based approach as the Vendi Score is it hints at the possibility of not needing to assume knowledge of species prevalence. Furthermore, the assumption of complete dissimilarity between species is strong and limits the Hill numbers' applicability. The similarity-based approach described above readily allows us to lift that assumption, by simply replacing the zero entries in the similarity matrix with a user-defined similarity function between species. 

Endowed with Theorem \ref{thm}, we can safely transition from Hill numbers---diversity indices that assume knowledge of species prevalence and that don't account for species similarity---to \emph{Vendi scores}, diversity indices that don't assume knowledge of species prevalence and that effectively account for species similarity. We denote by $\text{VS}_q(\rvx, \rvk)$ the Vendi score of order $q$ for the collection $\rvx$ under similarity function $\rvk(\cdot, \cdot)$ and define,
\begin{align}
    \text{VS}_q(\rvx, \rvk) &= \exp\left(\frac{1}{1-q} \log \sum_{i \in \text{supp}(\lambda(\rvx, \rvk))} \lambda_i^q(\rvx, \rvk)\right)
\end{align}
Here $\lambda(\rvx, \rvk)$ denotes the set of eigenvalues of the normalized similarity matrix induced by the input similarity function $\rvk(\cdot, \cdot)$, and $\text{supp}(\lambda(\rvx, \rvk))$ denotes the indices for the nonzero eigenvalues. 

Figure \ref{fig:shape-color} shows the behavior of the Vendi scores under different scenarios. The figure considers collections composed of elements with different shapes and colors. The scores are computed using a similarity function that assigns $1$ to elements with the same shape and color, $0.5$ to elements that have either the same shape or the same color, and $0$ to completely distinct elements. In the figure, we see that the Vendi Score ($q=1$), under class imbalance, fails to detect the introduction of a rare class (the black square), leading to a score of $\approx 3$ despite the presence of $4$ classes. The Vendi scores with orders smaller than $1$ are more sensitive to those rare classes and accurately measure diversity under class imbalance. The figure also shows Vendi scores with smaller orders to be more reliable under the presence of small variations within different classes. 

The Vendi scores have several desiderata beyond the ones we mentioned earlier. Indeed, they enjoy the same axioms as the Hill numbers and are therefore interpretable diversity scores that can be used to study diversity, e.g. in ecological systems. Some of these features include:
\begin{enumerate}
    \item The Vendi scores are monotonically decreasing as a function of the order $q$,
    \begin{align}
    \text{VS}_\infty(\rvx, \rvk) &\leq \dots \leq \text{VS}_1(\rvx, \rvk) \leq \text{VS}_0(\rvx, \rvk).
    \end{align}
    \item The Vendi score of order $2$ provides bounds on the Vendi score of order $\infty$: 
    \begin{align}
    \sqrt{\text{VS}_2(\rvx, \rvk)} &\leq \text{VS}_\infty(\rvx, \rvk) \leq \text{VS}_2(\rvx, \rvk). 
    \end{align}
\end{enumerate}

More importantly, the Vendi scores are differentiable, which makes them amenable to gradient-based methods in machine learning and science. This differentiability enables us to go beyond simply evaluating diversity to effectively enforcing diversity, by embedding the scores into objective functions of interest.

\textbf{Enforcing diversity.} Enforcing diversity has benefits in molecular simulations as shown by \citet{pasarkar2023vendi}, but also in many areas of machine learning, e.g. active learning and experimental design~\citep{maus2022discovering, nguyen2023nonmyopic}, generative modeling~\citep{dieng2019prescribed}, and reinforcement learning~\citep{eysenbach2018diversity}. 

Enforcing diversity with average similarity may be ineffective as average similarity fails to capture heterogeneity in data, e.g. variances between members of the same species. We refer the reader to \cite{friedman2022vendi} for examples illustrating the limitations of average similarity as a diversity metric. 

Enforcing diversity with other existing diversity metrics such as FID, KID, recall, and coverage is currently computationally impossible. Indeed, these metrics are either non-differentiable or may be challenging to optimize as they require querying large pre-trained networks at each optimization iteration. 

 The Vendi scores are differentiable interpretable diversity indices that can be used to effectively enforce diversity. In Section \ref{sec:empirical} we use the scores within Vendi Sampling to enforce diversity and study their effectiveness for accelerating molecular simulations. 

\begin{figure*}[t]
    \centering
    \begin{subfigure}[b]{0.375\textwidth}
        \centering
        \includegraphics[width=\linewidth]{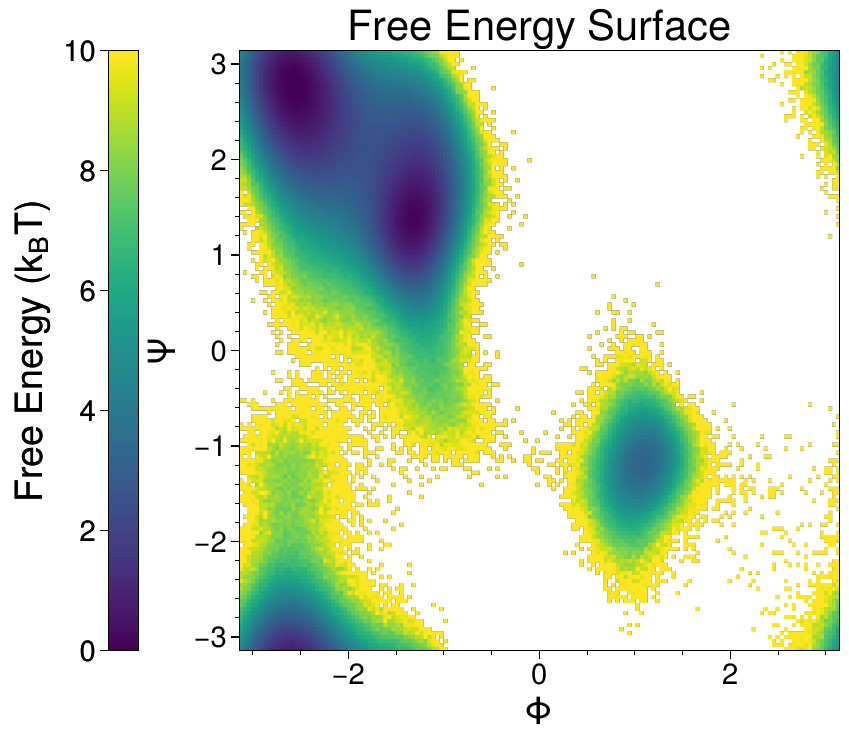}
\end{subfigure}
    \begin{subfigure}[b]{0.296\textwidth}
        \centering
        \includegraphics[width=\linewidth]{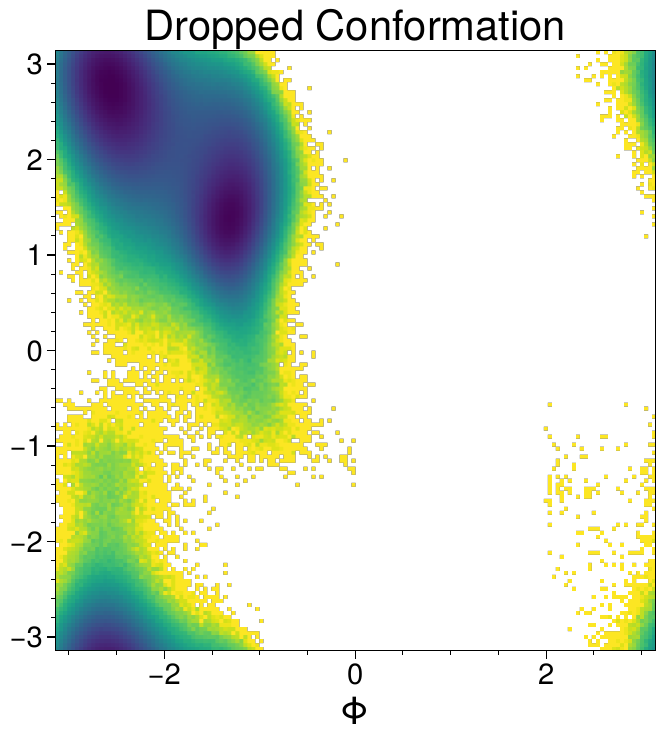}
\end{subfigure}
    \begin{subfigure}[b]{0.308\textwidth}
        \centering
        \includegraphics[width=\linewidth]{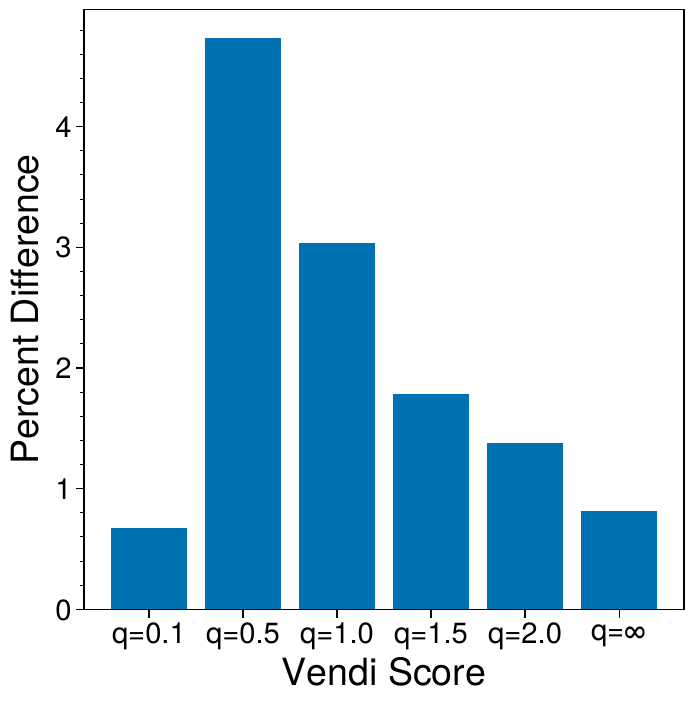}
\end{subfigure}
     \caption{\textbf{Sensitivity of different Vendi scores to missing Alanine Dipeptide conformation}. Left: Ramachandran plot from an unbiased simulation of Alanine Dipeptide plotted against the two dihedral angles $\phi,\psi$. Center: Ramachandran plot after removing the left-handed conformation. Right: The percent difference for different Vendi scores between samples from the original simulation and samples missing the left-handed state. Vendi scores are calculated using $20,000$ molecules from each set of samples using an invariant RBF Kernel with $\gamma=1$.}
    \label{fig:alamode}
\end{figure*}

\begin{figure*}[t]
    \centering
    \begin{subfigure}[b]{0.442\textwidth}
        \centering
        \includegraphics[width=\linewidth]{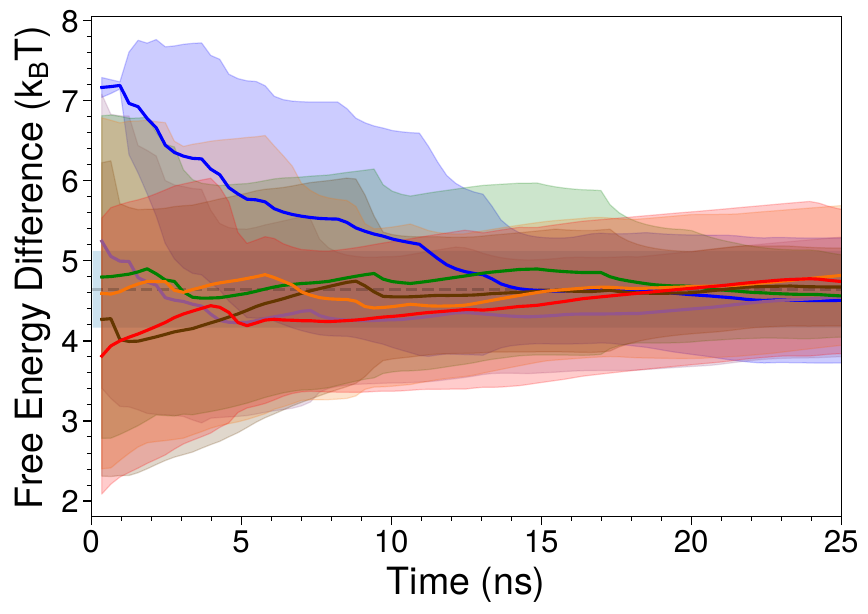}
\end{subfigure}
    \begin{subfigure}[b]{0.548\textwidth}
        \centering
        \includegraphics[width=\linewidth]{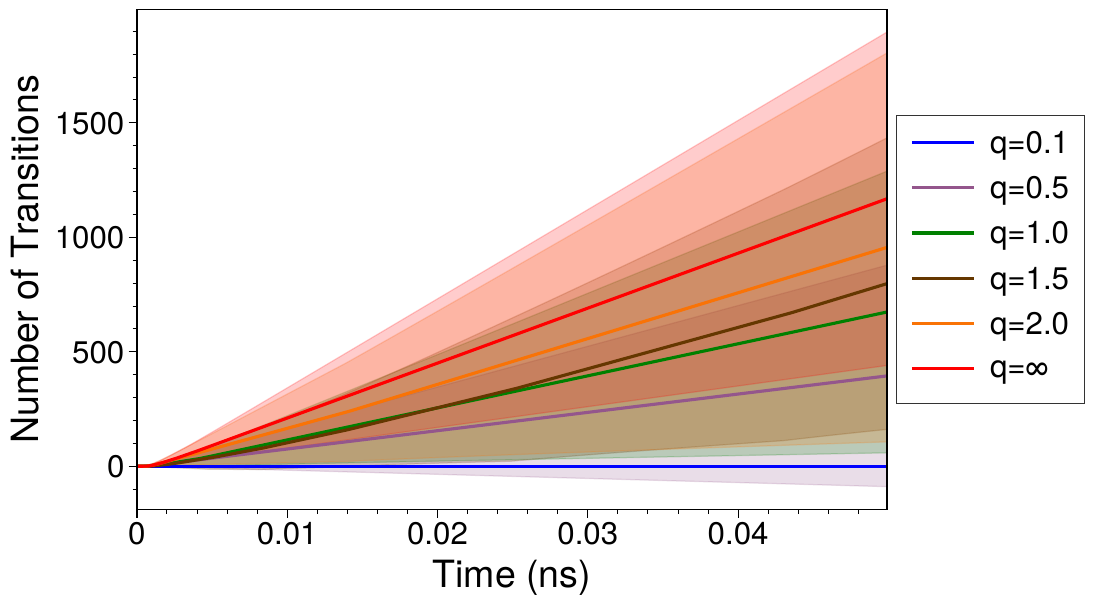}
\end{subfigure}
     \caption{\textbf{Behavior of the Vendi scores for sampling Alanine Dipeptide}. Left: Convergence of Vendi sampling under different scores over $25$ns of simulation to the free energy difference estimated from long unbiased simulations (dashed gray line). Right: Number of transitions for each score in and out of the left-handed state over the course of the first $50$ps of simulation. Shaded regions represent uncertainty over $10$ trials. 
    }
    \label{fig:alasamp}
\end{figure*}

\textbf{Computation.} Computing the Vendi scores requires finding the eigenvalues of an $N \times N$ normalized similarity matrix. This has complexity $\gO(N^3)$ which is computationally costly for large collections. Fortunately, Rayleigh-Risz provides a way to reduce computational cost. Consider a collection of size $N$ and denote by $\tilde{K}$ its normalized similarity matrix. Let $V \in \mathbb{R}^{N \times m}$ be an orthogonal matrix, with $m<<N$. We can compute the eigenvalues of $\tilde{K}$ by computing the eigenvalues of $V^*\tilde{K}V$, which is an $m \times m$ matrix. 

There are different ways to choose the orthogonal matrix $V$, each leading to a different scaling strategy. For very large collections, choosing $V$ to be a binary orthogonal matrix is equivalent to subsampling $m$ elements of the collection and approximating the Vendi scores using that subset. This would have $\gO(m^3)$ complexity, which is efficient for $m<<N$, and would allow to trade-off accuracy with computational cost since $m$ is determined by the user. When embeddings are readily available for the elements in the collection, e.g. Inceptionv3 or DINOv2 embeddings for images, we can perform a Gram-Schmidt orthogonalization of the embedding matrix of the elements of the collection to define $V$. We would then use V as described above to compute the Vendi scores. This would have complexity $\gO(N^2m)$---the same complexity as the computation of metrics such as FID---and has the benefit to extend the covariance trick mentioned in ~\cite{friedman2022vendi} to similarity functions beyond cosine similarity.

 \section{EMPIRICAL STUDY}
\label{sec:empirical}

\begin{figure*}[t]
    \centering
    \begin{subfigure}[b]{0.42\textwidth}
        \centering
        \includegraphics[width=\linewidth]{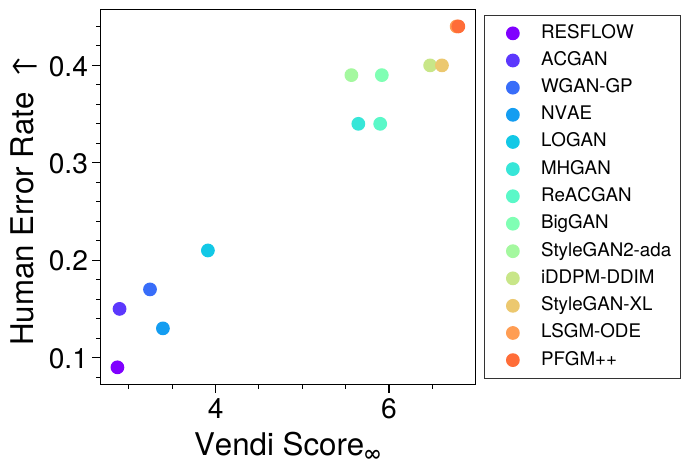}
\end{subfigure}
    \begin{subfigure}[b]{0.42\textwidth}
        \centering
        \includegraphics[width=\linewidth]{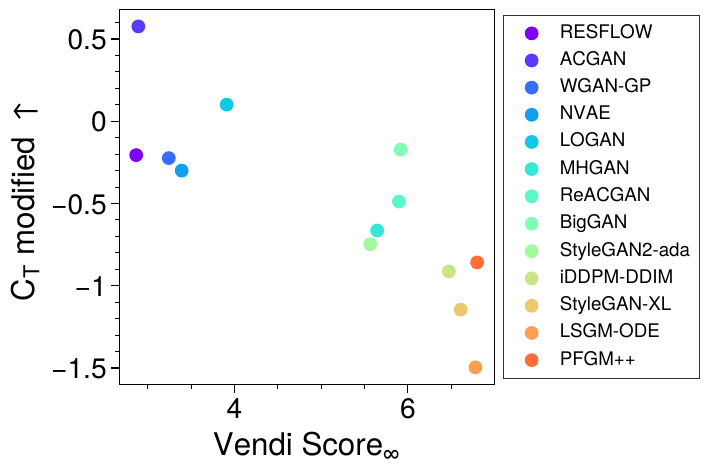}
\end{subfigure}\par
    \begin{subfigure}[b]{0.42\textwidth}
        \centering
        \includegraphics[width=\linewidth]{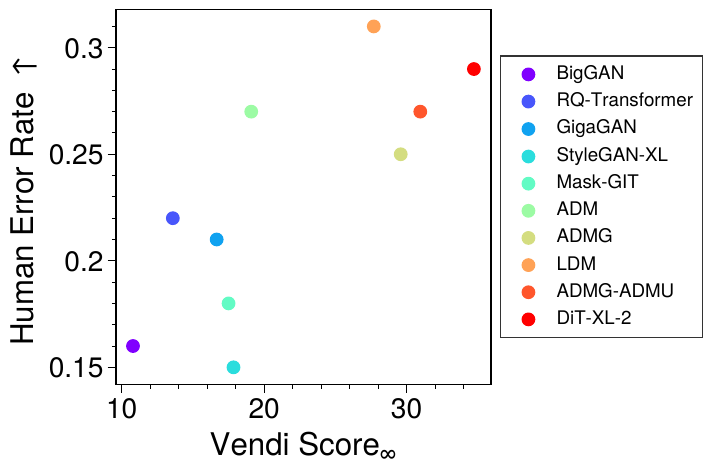}
\end{subfigure}
    \begin{subfigure}[b]{0.42\textwidth}
        \centering
        \includegraphics[width=\linewidth]{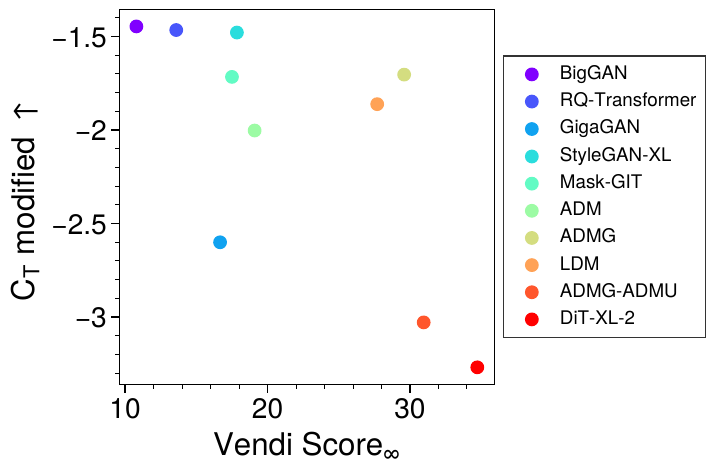}
\end{subfigure}
    \caption{\textbf{Vendi scores correlate strongly with human evaluation and memorization scores on CIFAR-10 and Imagenet256.} Left: Human classification error rate vs. Vendi Score$_\infty$ for models trained on CIFAR-10 (Top) and Imagenet256 (Bottom). Right: $C_T$-modified vs. Vendi Score$_\infty$ for models trained on CIFAR-10 (Top) and Imagenet256 (Bottom).}
    \label{fig:her_vs}
\end{figure*}

\subsection{Application To Vendi Sampling}
Molecular simulations through Langevin Dynamics are often plagued by slow mixing times between metastable states. A recent alternative approach, Vendi Sampling, was developed to improve the speed at which these simulations can be performed ~\citep{pasarkar2023vendi}. In Vendi Sampling, a collection of molecular replicas are evolved over time using Langevin Dynamics, with an additional diversity penalty term, called the \emph{Vendi force}, given by the gradient of the logarithm of the Vendi score. We aim to study how the choice of order $q$ affects the behavior of the Vendi force and the convergence of Vendi Sampling. We analyze convergence by looking at free energy differences. 

In order to provide an unbiased estimate of this quantity, we switch the coefficient of the Vendi force to $0$ after a specified number of steps. We only analyze samples taken when this coefficient is $0$. 

We can measure the relative probabilities of each state by performing a long unbiased simulation. We perform simulations in OpenMM (v. 8.0) \citep{eastman2017openmm}, following the experimental setup of ~\citet{pasarkar2023vendi}. 

To calculate the Vendi scores, we use a Gaussian Radial Basis Function (RBF) kernel $k(x,x')=\exp\left(-\gamma \|\mathbf{x} - \mathbf{x}'\|^2\right)$ where $\gamma$ is a hyperparameter of choice. We also require that the kernel be invariant to various rigid-body transformations, including translations and rotations. We follow the method outlined in \citet{jaini2021learning} for computing invariant coordinates. The invariant coordinates are passed into the RBF kernel, from which we can compute the Vendi scores. Further experimental details regarding how simulations are performed are available in the appendix. 

We first look to see how sensitive each Vendi score is in detecting Alanine Dipeptide conformations. In this molecule, the conformations are largely defined by its two dihedral angles $\phi$ (C-N-C$\alpha$-C) and $\psi$ (N-C$\alpha$-C-N) along the backbone. We focus on the left-handed state (defined by $0<\phi<2$), which constitutes $\approx 1\%$ of all samples in the reference simulations. We compare Vendi scores from samples from all conformation to samples that are not in the left-handed state. We find that for extreme values of $q$, the score is relatively unaffected, whereas for $q=0.5$ and $q=1$, there is a significant change in the Vendi score (Fig \ref{fig:alamode}). In Figure \ref{fig:shape-color}, $q=0.1$ can detect imbalanced classes, but it cannot detect when the intra-class variance changes between non-zero values. This result, combined with Fig \ref{fig:alamode}, suggests that the small values of $q$ can only detect rare classes when there is not a large amount of intra-class diversity, as there would be for Alanine Dipeptide conformations. Figure \ref{fig:shape-color} also demonstrates that Vendi Score$_\infty$ only detects the presence of large classes, so it is not surprising that it is insensitive to missing the small left-handed state. 

We further test the behavior of these scores in Vendi Sampling for Alanine Dipeptide (Fig \ref{fig:alasamp}). We evaluate convergence using the boundary of $\phi=0$. Hyperparameters are tuned for each choice of $q$ via grid search. We find that for most choices of order $q$, the sampling method converges within $0.4 k_BT$ of the estimated free energy difference within the first $5$ns of simulation, while the Vendi score with $q=0.1$ is slower to converge. Interestingly, unlike in the Double Well system (see Appendix), we find that $q=\infty$ is able to increase mixing rapidly in the initial stages of the simulation. With this choice of $q$, only the largest eigenvalue of the replica's kernel matrix is optimized, suggesting that the associated eigenvector is aligned with a useful biasing potential at various steps in the simulation. This highlights the importance of using large $q$ for regularization.

\subsection{Application To Generative Models}
\label{sec:imagegen}

\begin{figure}[t]
    \includegraphics[width=\linewidth]{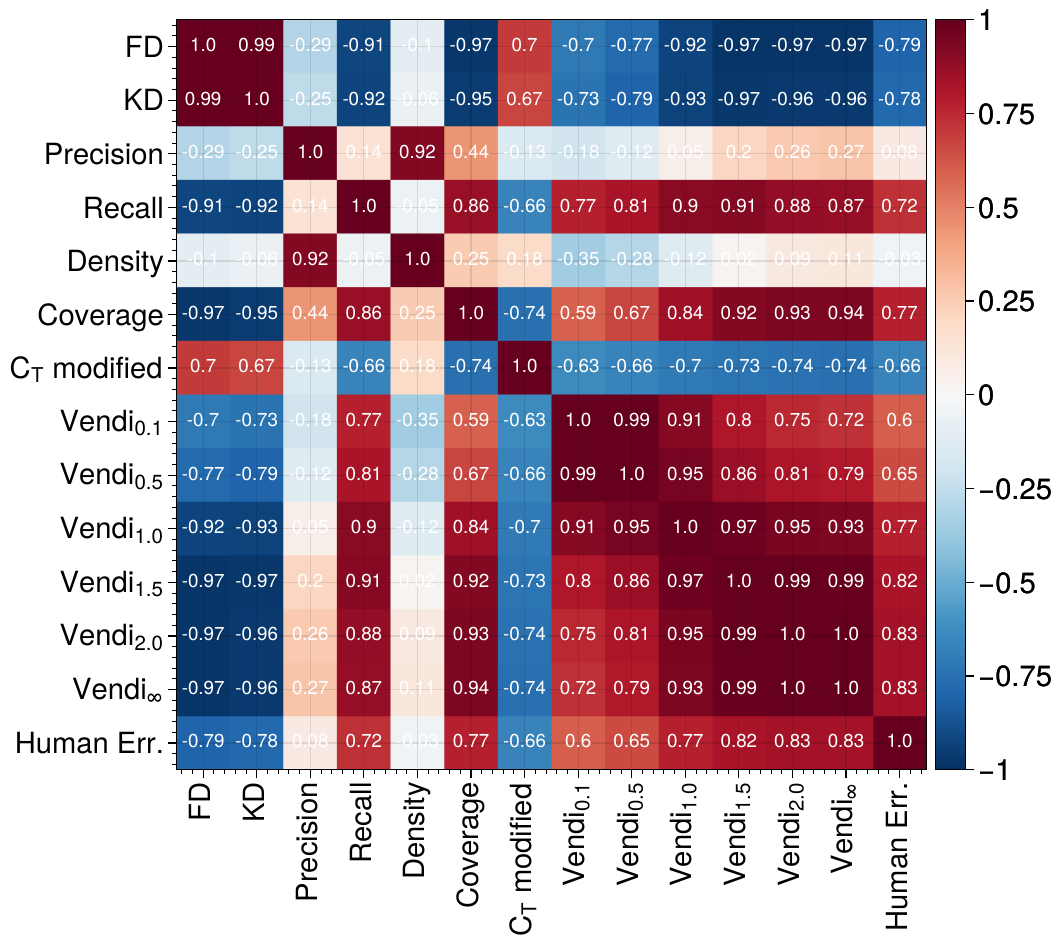}
    \caption{\textbf{Pearson correlations between metrics averaged across four training datasets}. $C_T$-modified is computed only on CIFAR-10 and Imagenet256. Vendi scores of large order $q$ correlate strongly with various metrics for evaluating generative models.}
    \label{fig:image_metrics}
\end{figure}

We analyze $40$ of the generative models presented in \citet{stein2023exposing}, which spans multiple classes of models and training datasets. In particular, we look at models trained on CIFAR-10 ~\citep{krizhevsky2009learning}, Imagenet256 ~\citep{deng2009imagenet}, LSUN-Bedroom ~\citep{yu2015lsun}, and FFHQ ~\citep{karras2019style}. Further description of the models and metrics is available in the appendix.  

\citet{stein2023exposing} find that the DINOv2 ViT-L/14 network ~\citep{oquab2023dinov2} produces a representation space for which various evaluation metrics align well with human evaluation. We thus use the same network to produce embeddings of the generated outputs from each model. Vendi Scores are computed on these embeddings using a linear kernel.

In Fig. \ref{fig:her_vs}, we see that Vendi Score$_{\infty}$ correlates quite well with a model's ability to produce high-quality images (its Human Error Rate), and $C_T$-modified, a memorization metric presented in \citet{stein2023exposing} that is a modified version of the original $C_T$ metric proposed by \citet{meehan2020non}. $C_T$-modified measures how often a generated data point is closer to the training data than the test data, penalizing models that are closer to the training data. Vendi Score$_{\infty}$ is most sensitive to large groups of similar samples, likely duplicates. The strong negative correlation between Vendi Score$_{\infty}$ and $C_T$-modified therefore suggests that the models with the highest Vendi Score$_{\infty}$ are producing large groups of samples around the training data. This also explains why the models with the highest Vendi Score$_{\infty}$ have high Human Error Rates as well: the samples being produced by the models contain many duplicates that are highly similar to the training data, which will make much of the generated output difficult to classify as fake.

Figure \ref{fig:image_metrics} provides a comprehensive overview of all tested metrics and their correlations. Notable in the figure is the strong negative correlation between Vendi Score$_{\infty}$ and the metrics used to measure sample quality, namely Fr\'echet Distance (FD) ~\citep{heusel2017gans, stein2023exposing}, and Kernel Distance (KD) ~\citep{binkowski2018gans, stein2023exposing}. Models with low FD and KD have high sample quality by definition. Due to the strong negative correlation between Vendi Score$_{\infty}$ and FD and KD, models with good sample quality as measured by FD and KD tend to produce duplicates, since they have high Vendi Score$_{\infty}$. This is the same conclusion we drew earlier, looking at human error rate and Vendi Score$_{\infty}$. We also take note of the correlation between the Vendi scores and coverage, another metric used to measure diversity ~\citep{naeem2020reliable}. Coverage measures how many training data points are 'close' to any generated data point. Models with a high Vendi Score$_{\infty}$ are producing images centered around the training data, which would satisfy the coverage requirement for those training samples.

\section{DISCUSSION}
\label{sec:discussion}

In this paper, we extended the Vendi Score~\citep{friedman2022vendi} to exhibit different levels of sensitivity to rare or common items in a collection. This led to a family of metrics, called \emph{Vendi scores}, indexed by an order $q \geq 0$. 
We observed that Vendi scores with small values of $q$ prioritize rarer elements, whereas those with high order $q$ emphasize more common items. 

\parhead{Choice of the order $q$.} The ideal choice of $q$ for a given setting depends on the phenomena under study. For example, in Figure \ref{fig:alamode}, we aimed to detect the presence/absence of a rare class when other larger classes with significant intra-class variance exist. In this case, a good value of $q$ is not as sensitive to the variance within a class but can also detect rare classes. This means $q$ cannot be too low, so the score can be somewhat insensitive to the intra-class variance, or too high, so the score can be somewhat sensitive to rare items. Using the orders $q=0.5$ or $q=1$ best balances these behaviors. It is worth noting that the choice of the kernel can influence this trade-off, as it will determine the amount of intra-class variance in the kernel matrix. 

In Vendi Sampling, the order $q$ must facilitate transitions over high energy barriers typical in molecular simulations. For example, in Alanine Dipeptide, the left-handed state is separated from the other states by a large energy barrier. The Vendi score with infinite order, VS$_\infty$, yielded the most transitions across this barrier. Since VS$_\infty$ only relies on the largest eigenvalue, it provides a bias potential along the axis corresponding to the associated eigenvector along which all transitions are occurring. However, in a simulation in which there are multiple transitions of interest, the eigenvector associated with the largest eigenvalue is likely insufficient, making smaller values of $q$ needed. 

When evaluating whether there are duplicates in the outputs of generative models, we want to use a Vendi score that is sensitive to duplication. Our results demonstrate that VS$_\infty$ is a good candidate for this task.

\parhead{Limitations.} This paper addressed the limitation of the Vendi Score under imbalanced settings. A pending problem is the choice of the kernel, which also affects the behavior of the Vendi scores. In future work, we aim to understand how the choice of kernel interfaces with the order $q$. \filbreak The Vendi scores can also be computationally costly to compute when faced with large collections of data that do not have vector representations. Finding methods for scaling the scores when no embeddings are available remains an open problem.

\section{CONCLUSION}

We extended the Vendi Score to a family of diversity metrics that allocate different levels of sensitivity to rare or common items in a collection. These scores vary in their overall behavior, such as their sensitivity to imbalanced classes and inter-class variance. Our molecular simulations of Alanine Dipeptide revealed that using a score of order $q = \infty$ enables faster mixing, suggesting that the associated eigenvector is aligned with a useful bias potential. We also demonstrated the utility of using the Vendi scores in evaluating image generative models. Our experiments revealed that image generative models that tend to score well on sample quality metrics, e.g. human error rate or Fréchet Distance, are those models that produce duplicates around memorized training samples. This calls for the need to pair sample quality metrics with the Vendi scores, to better distinguish models that have high sample quality only because of memorization and duplication around memorized samples and models that do produce sharp samples without memorization.

\subsection*{Acknowledgements}
Adji Bousso Dieng acknowledges support from the National Science Foundation, Office of Advanced Cyberinfrastructure (OAC) \#2118201, and from the Schmidt Futures AI2050 Early Career Fellowship. Amey Pasarkar is supported by an NSF-GRFP fellowship.

\subsection*{Dedication} This paper is dedicated to \href{https://en.wikipedia.org/wiki/Aline_Sitoe_Diatta}{Aline Sitoe Diatta}.

\clearpage
\bibliographystyle{apa}
\bibliography{arxiv}

\section{Appendix}

\subsection{Proof of Theorem 4.1}

\begin{theorem}[The Similarity-Eigenvalue-Prevalence Theorem]\label{thm:sim-eigen-prev}
    Let $(\rvx_1, \dots, \rvx_N)$ denote a collection of elements, where each $\rvx_i = (\rvx_{i1}, \dots, \rvx_{iM_i})$ contains a unique element repeated $M_i$ times, i.e. $\rvx_{ij} = \rvx_{ik}$ for all $j,k \in \left\{1, \dots, M_i\right\}$. Define $C = \sum_{i=1}^{N} M_i$. Let $\rmK \in \mathbb{R}^{C\times C}$ denote a kernel matrix such that $\rmK(\rvx_{i\bullet}, \rvx_{j\bullet}) = 1$ when $i = j$ and $0$ otherwise, $\forall i,j \in \left\{1, \dots, N\right\}$. Denote by $\tilde{\rmK} = \frac{\rmK}{C}$ the normalized kernel. Then $\tilde{\rmK}$ has exactly $N$ non-zero eigenvalues $\lambda_1, \dots, \lambda_N$ and $\lambda_i = \frac{M_i}{C}$ \text{ } $\forall i \in \left\{1, \dots, N\right\}$. 
\end{theorem}
\begin{proof}
     Without loss of generality we construct $\tilde{\rmK}$ as a block diagonal matrix with $N$ blocks, where each block corresponds to a matrix indexed by the elements of $\rvx_i$. Denote by $\rmJ_i$ the $i^{\text{th}}$ block.
     On the one hand, we have 
     \begin{align*}
         \text{det}(\tilde{\rmK}) &= \prod_{i=1}^{N} \text{det}(\rmJ_i) \quad \text{and} \quad \text{det}(\tilde{\rmK} - \gamma \rmI_C) = \prod_{i=1}^{N} \text{det}(\rmJ_i - \gamma \rmI_{M_i})
     \end{align*}
     for any $\gamma$. Therefore the eigenvalues of $\tilde{\rmK}$ are exactly the collection of the eigenvalues of $\rmJ_1, \dots, \rmJ_N$. On the other hand, each $\rmJ_i$ is of size $M_i \times M_i$ and we have $\rmJ_i = \frac{1}{C}\left(1 \dots 1\right)^T\left(1\dots 1\right)$. Therefore rank($\rmJ_i$) $= 1$ and the null space of $\rmJ_i$ is of dimension $M_i-1$. This means $\rmJ_i$ has $M_i-1$ zero eigenvalues. Denote by $\lambda_i$ the remaining eigenvalue and by $\rvv_{i}$ its associated eigenvector. We have 
     \begin{align*}
         \rmJ_i (\rvv_{i1} \dots \rvv_{iM_i})^T &= \lambda_i (\rvv_{i1} \dots \rvv_{iM_i})^T\\
         \frac{1}{C}\left(1 \dots 1\right)^T\left(1\dots 1\right)(\rvv_{i1} \dots \rvv_{iM_i})^T &= \lambda_i (\rvv_{i1} \dots \rvv_{iM_i})^T\\
         \frac{1}{C}\left(1 \dots 1\right)^T\left(\rvv_{i1}+ \dots +\rvv_{iM_i}\right) &= \lambda_i (\rvv_{i1} \dots \rvv_{iM_i})^T
     \end{align*}
     Then $\rvv_i = \left(1 \dots 1\right)$ and $\lambda_i = \frac{M_i}{C}$. Since the eigenvalues of $\tilde{\rmK}$ correspond to the eigenvalues of $\rmJ_1, \dots, \rmJ_N$, we conclude $\tilde{\rmK}$ has exactly $N$ nonzero eigenvalues $\lambda_1, \dots, \lambda_N$ and $\lambda_i= \frac{M_i}{C}$ $\forall i$. 
\end{proof}
\subsection{Vendi Sampling: Alanine Dipeptide Experimental Details}
We compare against unbiased Alanine Dipeptide Langevin Dynamics simulations. These simulations used a time step of $2.0\,\mathrm{fs}$ and a collision frequency of  $1.0\mathrm{ps}^{-1}$. We establish baselines by running 10 simulations of $100\,\mathrm{ns}$ with $32$ replicas. 

Vendi Sampling requires the computation of the Vendi Score at each simulation time-step. In our molecular studies, we compute the Vendi Score using the translation- and rotation-invariant kernel defined in \cite{jaini2021learning}. This kernel takes as input the $3$D coordinates of the molecular replicas, produces a set of invariant coordinates, and then applies a Radial Basis Function (RBF) kernel on the coordinates. Invariant coordinates are computed by first centering each molecule at the origin (achieving translation invariance) and then aligning all molecules along a common frame (achieving rotational invariance). We compute the Vendi Score on the resulting similarity matrix from the invariant kernel. 

We analyze convergence by measuring the dihedral angle $\phi$ of each Alanine Dipeptide sample. After collecting all angles $\phi$ in a given simulation, we compute the free energy difference between samples with $\phi<0$ and $\phi>0$. Following \cite{pasarkar2023vendi}, this is calculated with 
\begin{align}
    F = -\log \frac{P(\phi>0)}{\log P(\phi<=0)}
\end{align}
This difference is computed for each of the $10$ baseline trials to estimate the true free energy difference across this boundary.  We also use this boundary to compute the number of times that a replica transitions in and out of the left-handed states. 

\begin{figure*}[t]
    \centering
    \begin{subfigure}[b]{0.31\textwidth}
        \centering
        \includegraphics[width=\linewidth]{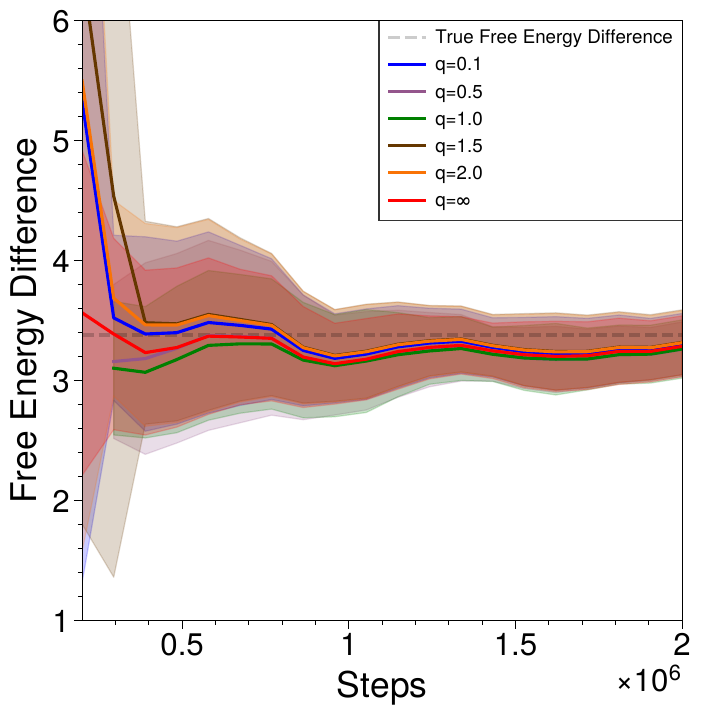}
\end{subfigure}
    \begin{subfigure}[b]{0.333\textwidth}
        \centering
        \includegraphics[width=\linewidth]{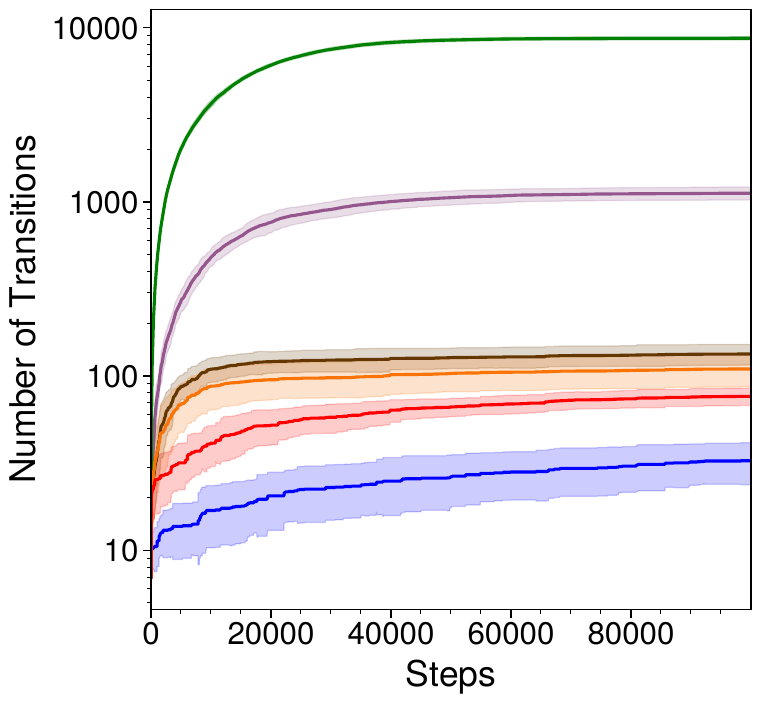}
\end{subfigure}
    \begin{subfigure}[b]{0.33\textwidth}
        \centering
        \includegraphics[width=\linewidth]{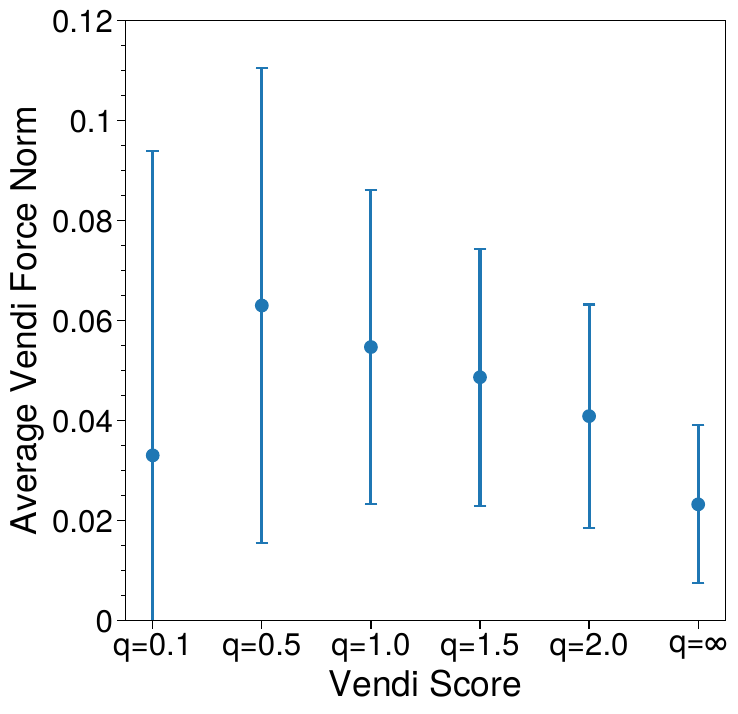}
\end{subfigure}
     \caption{\textbf{Performance of Different Vendi Scores in Sampling from Double Well} Left: Free energy difference over time for each choice of Vendi Score shows similar levels of convergence to the true free energy difference. Hyperparameters are tuned individually for each Vendi Score. Center: Number of transitions of replicas across boundary of $x=0$ over time for each choice of Vendi Score. Hill Numbers farther from $1$ seem to provide less transitions, likely due to smaller gradients, as shown in Right.}
    \label{fig:DW}
\end{figure*}

\subsection{Vendi Sampling: Double Well System}
\label{appendix:DoubleWell}
We additionally study the two-dimensional Double Well system from \cite{noe2019boltzmann}. The Double Well system is challenging for Langevin dynamics due to a large energy barrier that separates two imbalanced modes. Through the addition of the Vendi force, we expect to see fast convergence as well as transitions across this barrier. We perform this set of experiments following the setup used in ~\citet{pasarkar2023vendi}. 

To compute the Vendi Score, we used the kernel $k(x,x') = 1-\frac{|x-x'|}{|x|+|x'|}$. We also use a linear annealing schedule for $\nu$, decreasing it at a constant rate to $0$ for a specified period of time. For each choice of Vendi Score, we determined the optimal hyperparameters using a grid search. In Figure \ref{fig:DW}, we used the following hyperparameters: For $q=0.1$ and $q=\infty$, we used $\nu=50$ with annealing rate $\frac{1}{50000}$. For $q=0.5$ and $q=1.$, we use $\nu=100$ with annealing rate $\frac{1}{100000}$. And finally for $q=1.5$ and $q=2.$, we used $\nu=50$ with annealing rate $\frac{1}{25000}$. $16$ particles are initialized with random positions sampled from $U[-2.5,2.5]^2$ and simulations are performed with a step-size of $10^{-2}$ for $2,000,000$ million steps. We measure convergence using the Free energy difference between the regions $\left\{x \in [-2.5,0], y \in [-4, 4] \right\}$ and $\left\{x \in [0,2.5], y \in [-4, 4] \right\}$.

The choice of Vendi Score does not noticeably affect convergence in this system, but the Vendi Score regularization is indeed quite different across scores (Fig. \ref{fig:DW}). Over the first $100,000$ steps, we find that for Vendi Scores with extreme Hill Numbers, $q=0.1$ and $q=\infty$, there is still a slow transition rate for particles across the boundary while the Vendi force is active compared to other choices of $q$. The slow transition rate is supported by the small Vendi force magnitudes for $q=0.1$ and $q=\infty$. 

We have observed that $q=0.1$ leads to a very sensitive Vendi Score, whereas $q=\infty$ gives the least sensitive Vendi Score. Yet, they provide similar effects in the Double Well setting: for $q=0.1$, samples are likely to already be considered diverse and therefore there is not much optimization necessary through the Vendi force. For large $q$, the score is determined only by the largest eigenvalue of the gram matrix $K/n$. Optimizing for the largest eigenvalue may not be informative in some systems. 

Meanwhile, for $q=0.5$ and $q=1.0$, the effect of the Vendi Force is largest, demonstrating a trade-off between the sensitivity of small and large Hill Numbers.  

\begin{figure*}[t]
    \centering
    \begin{subfigure}[b]{0.42\textwidth}
        \centering
        \includegraphics[width=\linewidth]{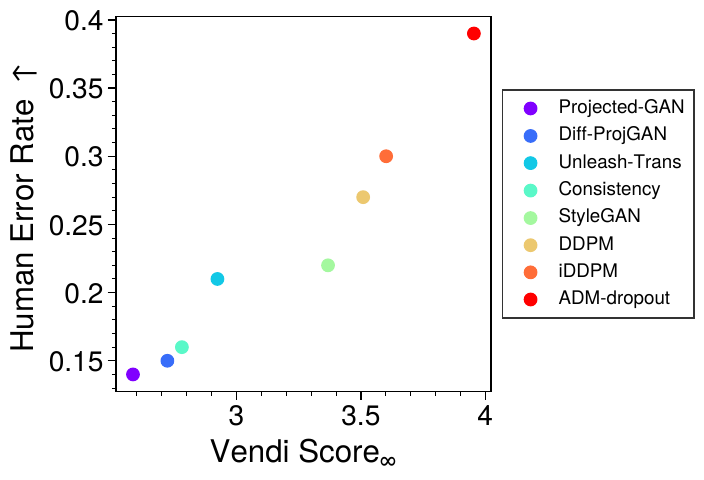}
\end{subfigure}
    \begin{subfigure}[b]{0.42\textwidth}
        \centering
        \includegraphics[width=\linewidth]{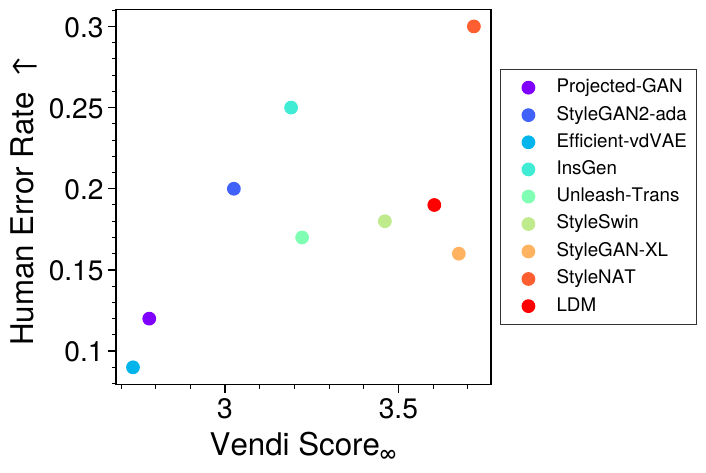}
\end{subfigure}
    \caption{Vendi Score$_\infty$ is well-correlated with human evaluation on LSUN-Bedroom (Left) and FFHQ (Right) datasets.}
    \label{fig:err_rate}
\end{figure*}

\subsection{Image Generative Model Analysis}
\label{app:images}
\citet{stein2023exposing} provided analysis of dozens of image generative models across datasets and model types. We study all models for which they provided publically available image outputs. For details regarding dataset curation and model training, we refer the reader to \citet{stein2023exposing}. 

For CIFAR-10, there were $6$ StudioGAN models used \citep{kang2023studiogan}: ACGAN \citep{odena2017conditional}, BigGAN \citep{brock2018large}, LOGAN \citep{wu2019logan}, ReACGAN \citep{kang2021rebooting}, MHGAN \citep{turner2019metropolis}, and WGAN-GP \citep{gulrajani2017improved}. Other models tested included LSGM-ODE \citep{vahdat2021score}, iDDPM-DDIM \citep{nichol2021improved}, PFGM++ \citep{xu2023pfgm++}, RESFLOW \citep{chen2019residual}, NVAE \citep{vahdat2020nvae}, StyleGAN2-ada \citep{Karras2020ada}, StyleGAN2-XL \citep{sauer2022stylegan}. 

For Imagenet, we analyzed results from the following models: ADM, ADMG, ADMG-ADMU \citep{dhariwal2021diffusion}, BigGAN \citep{brock2018large}, DiT-XL-2, GigaGAN \citep{kang2023scaling}, LDM \citep{rombach2022high}, Mask-GIT \citep{chang2022maskgit}, RQ-Transformer \citep{lee2022autoregressive}, and StyleGAN-XL \citep{sauer2022stylegan}.

For FFHQ, we used the following models: Efficient-vdVAE \citep{hazami2022efficientvdvae}, Insgen \citep{yang2021data}, LDM \citep{rombach2022high}, Projected-GAN \citep{sauer2021projected}, StyleGAN2-ada \citep{Karras2020ada}, StyleGAN2-XL \citep{sauer2022stylegan}, StyleNAT \citep{walton2022stylenat}, StyleSwin \citep{zhang2022styleswin}, and Unleashing-Transformers \citep{bond2022unleashing}.

Finally, for LSUN-Bedroom, we use the following models: Unleashing-Transformers \citep{bond2022unleashing}, Projected-GAN \citep{sauer2021projected}, ADMNet-dropout \citep{dhariwal2021diffusion}, DDPM \citep{ho2020denoising}, iDDPM \citep{nichol2021improved}, StyleGAN \citep{karras2019style}, Diffusion-projected GAN \citep{wang2022diffusion}, and Consistency \citep{meehan2020non}. 

We also show that the Human Error Rate is strongly correlated with Vendi Score$_\infty$ on the LSUN-Bedroom and FFHQ datasets in Figure \ref{fig:err_rate}.

 \end{document}